\documentclass[twoside]{article}

\usepackage[accepted]{aistats2023}
\usepackage[utf8]{inputenc} 
\usepackage[T1]{fontenc}    
\usepackage{hyperref}       
\usepackage{url}            
\usepackage{booktabs}       
\usepackage{amsfonts}       
\usepackage{nicefrac}       
\usepackage{microtype}      
\usepackage{dsfont}
\usepackage{caption}
\usepackage{amsmath}
\DeclareMathOperator*{\argmax}{argmax}
\usepackage{bm}
\usepackage{todonotes}
\usepackage{graphicx}
\usepackage{xcolor}
\usepackage{colortbl}
\usepackage{subcaption}
\usepackage{float}
\usepackage{verbatimbox}
\usepackage{booktabs}
\usepackage{natbib}
\usepackage{amsthm}
\usepackage{thm-restate}
\newtheorem{theorem}{Theorem}

\newtheorem{corollary}{Corollary}

\usepackage{enumitem}
\usepackage[algoruled,vlined,nofillcomment]{algorithm2e}
\usepackage{mathtools}

\newcommand{\f}{\textbf{f}}
\newcommand{\m}{\textbf{m}}
\newcommand{\x}{\textbf{x}}
\newcommand{\y}{\textbf{y}}
\newcommand{\z}{\textbf{z}}
\newcommand{\uz}{\textbf{u}}
\newcommand{\calX}{\mathcal{X}}
\newcommand{\reals}{\mathds{R}}
\newcommand{\data}{\mathcal{D}}
\newcommand{\N}{\mathcal{N}}
\newcommand{\KX}{K_X}
\newcommand{\kX}{\textbf{k}_X}
\newcommand{\KZ}{K_Z}
\newcommand{\kZ}{\textbf{k}_Z}

\begin{document}

%
\runningtitle{Inducing Point Allocation for Sparse Gaussian Processes in High-Throughput Bayesian Optimisation}

%

\twocolumn[

\aistatstitle{Inducing Point Allocation for Sparse Gaussian Processes in High-Throughput Bayesian Optimisation}

\aistatsauthor{ Henry B. Moss \And Sebastian W. Ober \And  Victor Picheny }

\aistatsaddress{ Secondmind.ai \And Secondmind.ai \And Secondmind.ai } ]

\begin{abstract}
Sparse Gaussian processes are a key component of high-throughput Bayesian optimisation (BO) loops; however, we show that existing methods for allocating their inducing points severely hamper optimisation performance. By exploiting the quality-diversity decomposition of determinantal point processes, we propose the first inducing point allocation strategy designed specifically for use in BO. Unlike existing methods which seek only to reduce global uncertainty in the objective function, our approach provides the local high-fidelity modelling of promising regions required for precise optimisation. More generally, we demonstrate that our proposed framework provides a flexible way to allocate modelling capacity in sparse models and so is suitable for a broad range of downstream sequential decision making tasks.
\end{abstract}

\section{Introduction}

Countless design tasks in science, industry and machine learning can be formulated as high-throughput optimisation problems, as characterised by access to substantial evaluation budgets and an ability to make large batches of evaluations in parallel. Prominent examples include high-throughput screening within drug discovery \citep{hernandez2017parallel}, DNA sequencing, and experimental design pipelines, where automation allows researchers to efficiently oversee thousands of scientific experiments, field tests and simulations through sensor arrays and cloud compute resources \citep{kandasamy2018parallelised}. However, such design tasks tend to have large search spaces and multi-modal optimisation landscapes such that, even under large optimisation budgets, only a small proportion of candidate solutions can ever be evaluated, and often only with significant levels of observation noise. Consequently, most existing optimisation routines are unsuitable, as brute-force methods require too many evaluations.

\begin{figure}%
\begin{center}
    \includegraphics[width= 0.95\columnwidth]{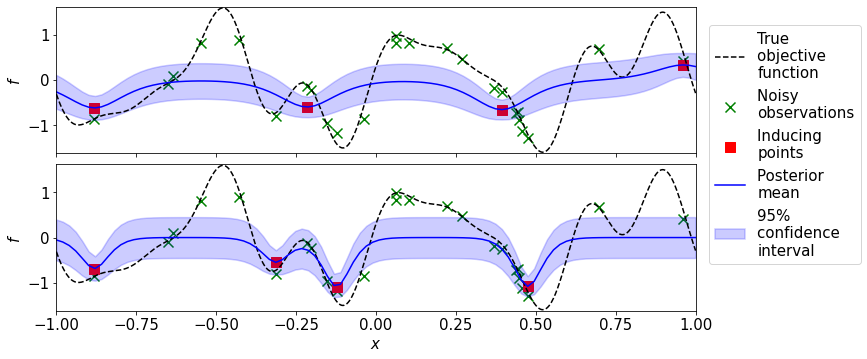}%
\end{center}
\caption{A toy problem showing two sparse GP surrogate models, one with its inducing points chosen using an existing method (top) and the other using one of our proposed BO-specific methods (bottom) that focuses modelling resources into promising areas of the search space.  Our model is better suited for assisting BO find this function's minima.
}%
\label{fig::demo}%
\end{figure}

Bayesian optimisation \citep[BO, see][for a review]{shahriari2016taking} has surfaced as the \textit{de facto} approach for solving noisy black-box optimisation tasks under restricted evaluation budgets, with numerous successful applications across the empirical sciences and industry. However, vanilla BO relies on Gaussian processes \citep[GPs,][]{rasmussen2006gaussian},
which incur a significant computational overhead for each individual optimisation step. This cost becomes increasingly unwieldy as data volumes increase,
making it unsuitable for the high-throughput tasks motivated above. 

Several ways to scale up BO with large data volumes have been explored, including using local models \citep{eriksson2019scalable} or neural networks \citep{hernandez2017parallel}.
Among those alternatives, using sparse GPs \citep{titsias2009variational} are particularly attractive as they dramatically reduce the computational cost of GPs and have enabled BO to be applied to a range of applications including molecular search \citep{griffiths2020constrained}, laser optimisation \citep{mcintire2016sparse}, model optimisation \citep{nickson2014automated}, alloy design \citep{yang2021sparse}, and risk-adverse optimisation \citep{picheny2022bayesian}. 

In a nutshell, sparse GPs replace the full set of observations by a smaller representative set of pseudo-observations referred to as \textit{inducing points}.
The choice of the inducing point locations has a critical influence on the behaviour of the model, as it encodes local expressivity.
However, existing approaches for inducing point allocation (IPA) focus purely on regression tasks, i.e., the global accuracy of models, and so sacrifice high-fidelity (local) modelling of promising regions which is required, as confirmed by our experiments, for effective optimisation (Figure \ref{fig::demo}). For this reason, there is a need for BO-specific IPA strategies; however, to our knowledge, no such methods exist in the literature.

Our contributions can be summarised as follows:
\begin{enumerate}
  \setlength\itemsep{0.1em}
    \item We demonstrate that existing IPA strategies do not support high-precision BO.
    \item We introduce the use of quality-diversity decomposed DPPs as an IPA, allowing the trade-off of an IPA's diversity against an underlying preference.
    \item We propose a guide for practical BO-specific IPA methods along with several specific recommendations. 
    \item We show that our methods out-perform established baselines across synthetic and real-world high-throughput optimisation and active learning tasks. 
\end{enumerate}


\section{Background}
\label{sec:background}

\textbf{Bayesian Optimisation.} BO is a highly data-efficient method for finding the optima of a smooth function $f:\mathcal{X}\rightarrow\mathds{R}$. By using a probabilistic surrogate model, typically a GP, coupled with a data acquisition strategy, evaluations are focused into promising areas of the search space $\mathcal{X}$, allowing identification of good solutions within heavily constrained evaluation budgets.

Popular examples of data acquisition strategies include expected improvement \citep[EI,][]{jones1998efficient}, knowledge gradient \citep{frazier2008knowledge}, entropy search \citep{hennig2012entropy}, or 
Thompson sampling \citep[TS,][]{kandasamy2018parallelised}. 
While our framework is not specific to any acquisition strategy, we focus mainly on Thompson sampling, a simple yet effective strategy that evaluates the maxima (minima) of random samples from the surrogate model when performing black-box maximisation (minimisation). TS is an obvious choice for high-throughput BO due to its natural ability to handle highly parallelised optimisation resources, 
e.g. for molecular search \citep{hernandez2017parallel} or distributed computing \citep{kandasamy2018parallelised}. 
Moreover, \citet{vakili2021scalable} have recently shown that the decoupled TS approach of \citet{wilson2020efficiently} can provide a drastic efficiency gain over traditional TS without significant impact on regret performance.

\textbf{Gaussian Processes.} GP models are a popular choice as surrogate models for BO, as they combine flexibility with reliable uncertainty estimates.
A GP can be defined as an infinite collection of random variables, any finite number of which are distributed according to a multivariate Gaussian \citep{rasmussen2006gaussian}.
Consider a dataset $\data = (X, \y)$ consisting of $N$ input-output pairs $(\x_n, y_n)$, where $\x \in \calX$ and $y \in \reals$.
In Gaussian process regression, we model this dataset as a noisy realization of a latent function,
\begin{align*}
    y_n &= f(\x_n) + \epsilon_n, \quad \epsilon \sim \N(0, \sigma^2),
\end{align*}
where we have given $f$ a GP prior, $f \sim \mathcal{GP}(\mu_0(\cdot), k(\cdot, \cdot))$, and $\sigma^2$ is the noise variance.
$\mu_0: \calX \rightarrow \reals$ is the (prior) mean function, whereas $k: \calX \times \calX \rightarrow \reals$ is a positive semidefinite covariance function or kernel; taken together, these are sufficient to fully describe the GP prior, which states that $f(X) \sim \N(\mu_0(X), \KX)$, where we have defined $\KX \coloneqq \left[k(\x, \x')\right]_{\x, \x' \in X}$ (abusing the notation slightly).
For notational simplicity, we henceforth assume the mean function to be zero.
By conditioning on the observed data, we can compute the exact posterior $p(f|\y)$ as a GP with mean and covariance functions
\begin{align}
\label{eq:gp-predict}
    \mu(\x) &= \kX(\x)^T(\KX + \sigma^2 I_N)^{-1}\y \\
    \nonumber
    \Sigma(\x, \x') &= k(\x, \x') - \kX(\x)^T (\KX + \sigma^2 I_N)^{-1}\kX(\x'),
\end{align}
where we have defined $\kX \coloneqq \left[k(\x', \x) \right]_{\x' \in X}$ and the identity matrix $I_N \in \reals^{N \times N}$.
While we can compute the exact posterior predictive using these equations, in practice we are often limited to using small datasets, as computing the required $(\KX + \sigma^2 I_N)^{-1}$ requires $O(N^3)$ computational complexity and $O(N^2)$ memory.


\textbf{Sparse Variational Gaussian Processes.} To mitigate the computational cost of GP modelling and allow for larger datasets, sparse variational approaches \citep{titsias2009variational, hensman2013gaussian} have been developed.
Instead of conditioning on the $N$ training points, sparse GPs learn a set of $M << N$ \emph{inducing variables} $\uz \in \reals^M$, defined at \emph{inducing locations} $Z = \{\z_m\}_{m=1}^M, \z_m \in \calX$, so that $\uz = f(Z)$.
By defining an approximate posterior over the inducing variables $q(\uz) = \N(\uz; \m, S)$ with variational parameters $\m \in \reals^M, S \in \reals^{M \times M}$, we can simultaneously learn the inducing locations and variational parameters by maximizing the \emph{evidence lower bound (ELBO)}:
\begin{align*}
    \mathcal{L} = \mathbb{E}_{q(\f)}\left[\log p(\y| \f)\right] - \mathrm{KL}(q(\uz) || p(\uz)),
\end{align*}
where $q(\f) = \N(\f; \mu_\f, \Sigma_\f)$ is the approximate posterior over the function values defined at the data points implied by conditioning on $\uz$,
\begin{align*}
    \mu_\f &= \kZ(X)^T\KZ^{-1}\m \\
    \Sigma_\f &= \KX + \kZ(X)^T \KZ^{-1}(S - \KZ)\KZ^{-1}\kZ(X),
\end{align*}
where we have defined $\kZ(\cdot)$ and $\KZ$ analogously to $\kX(\cdot)$ and $\KX$, respectively.
We refer to this model as the \emph{sparse variational Gaussian process (SVGP)}.
The SVGP model requires $O(M^2 \tilde{N})$ computational complexity and $O(M\tilde{N})$ memory, where $\tilde{N}$ is the size of a minibatch, a significant saving over the exact GP.


While the inducing locations can be learned according to the ELBO along with the model hyperparameters and variational parameters, 
\citet{burt2020convergence} argues that this yields a very challenging high-dimensional and non-convex optimization task with a complicated dependence structure that is difficult to solve, converges slowly and often provides sub-optimal models.
Whereas for regression this may be allowable with sufficient computational resources, for BO we must be able to reliably and quickly fit models, and therefore it would be preferable to allocate $Z$ \emph{a priori} and keep them fixed.
Moreover, optimizing $Z$ according to the ELBO encourages the inducing points to approximate the posterior globally \citep{matthews2016sparse}, which we will argue is wasteful for BO applications.
Therefore, we focus the remainder of our work on methods for inducing point allocation (IPA), which we will use to set the inducing points at the start of each BO step.
We start by describing prior work for IPA, which focuses on regression, before moving to our BO-specific IPA contributions.

\begin{figure*}%
\subfloat[Exact GP]{\includegraphics[height= 0.31\textwidth]{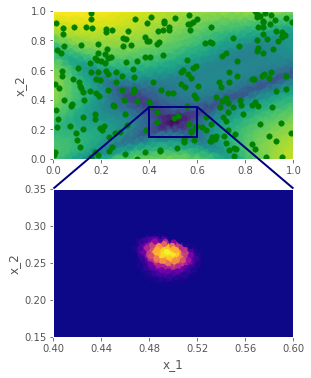}}%
\subfloat[SVGP with K-means]{\includegraphics[height= 0.31\textwidth]{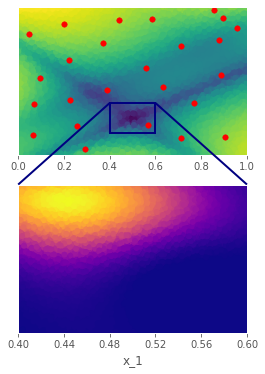}}%
\subfloat[SVGP with CVR]{\includegraphics[height= 0.31\textwidth]{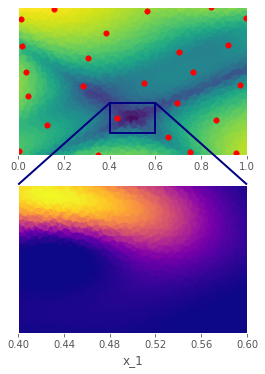}}%
\subfloat[SVGP with IMP-DPP ]{\includegraphics[height= 0.31\textwidth]{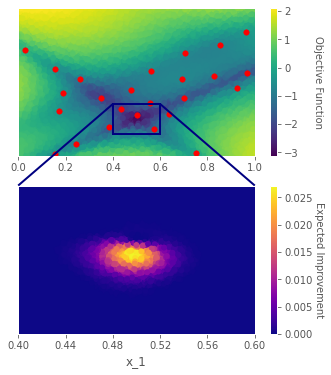}\label{subfig:CIR}}%
\caption{ The top row shows (a) 250 available training data points (green) alongside (b,c,d) three different 25 point IPAs (red) chosen for a function minimisation task. Existing approaches which (b) use the centroids from a k-means clustering of the available data or (c) use the CVR strategy provide balanced coverage of the whole search space. In contrast, our IMP-DPP strategy (d) focuses modelling resources into promising central areas. The bottom rows show expected improvement acquisition functions evaluated in the promising region according to (a) an exact GP trained on all available data and those (b,c,d) arising from SVGPs with the IPAs above. Of the SVGPs, only our proposed IMP-DPP's acquisition function agrees with the exact GP.
}%
\label{fig::inducing_points}%
\end{figure*}

\section{Inducing Point Allocation for Regression}
\label{sec:regression}

Existing IPA strategies include taking a random subset of the data, sampling uniformly across the problem's search space, or using centroids obtained by running a K-means algorithm on the data \citep{hensman2013gaussian}. The remainder of this Section details the recent DPP-based method of \cite{burt2019rates}, laying out important groundwork for our proposed BO-specific IPA strategies. 

\textbf{Determinantal Point Processes.} For regression tasks, a meaningful criterion for IPA would be to have the points spread as uniformly as possible across the input data $X$.
It would also be sensible to have a criterion that takes the kernel and its hyperparameters into account.
\citet{burt2020convergence} showed that one way of achieving these is by using an $M$-determinantal point process \citep[$M$-DPP,][]{kulesza2012determinantal}.
An $M$-DPP chooses the $M$ points in $Z$ by sampling them from the data $X$ with probability proportional to the determinant of the Gram matrix $\KZ$:
\begin{align}
    \mathds{P}(\mathcal{Z}=Z) \propto \big|\KZ\big|.
    \label{eq:DPP}
\end{align}
Notice that this criterion meets our two criteria described above: 1) if two points are close together in $Z$, the determinant will typically be small since the kernel will have high covariance for those points, giving the $M$-DPP repulsive properties so that the selected points have a uniform spread, and 2) the determinant clearly depends on the kernel.
Using results from the $M$-DPP literature, \citet{burt2020convergence} was able to show that sampling inducing points in this way from a DPP will lead to a small expected KL divergence between approximate and true posteriors, $KL[q(f) || p(f|\y)]$.
Moreover, these results have recently been used to prove regret bounds in BO for sparse GP methods \citep{vakili2021scalable}.


\textbf{Conditional Variance Reduction.} In practice, sampling from a DPP is computationally expensive.
Therefore, \citet{burt2020convergence} suggests finding the \emph{maximum a posteriori (MAP)} estimate of a DPP, i.e., finding the set of inducing points $Z$ with maximum probability according to Eq.~\ref{eq:DPP}.
While exact MAP estimation of a DPP is known to be NP-hard \citep{ko1995exact}, \citet{chen2018fast} provides an algorithm for approximate MAP estimation in $O(M^2N)$, which \citet{burt2020convergence} uses in practice.
This algorithm greedily builds its set of points $Z$ by choosing the $j^{th}$ point from $X \setminus Z_{1:j-1}$ as
\begin{align}
    \textbf{z}_j = \argmax_{\textbf{z}\in X \setminus Z_{1:j-1}} \big|K_{Z_{1:j-1}\bigcup \{\textbf{z}\}}\big|.
    \label{eq:DPP_greedy}
\end{align}


Interestingly, this DPP-based IPA strategy (\ref{eq:DPP_greedy}) is equivalent to greedily building a set of inducing points by maximising the posterior predictive variance of a noise-free GP model $f\sim \mathcal{GP}(0, k)$ conditioned on previously selected observations, i.e., choosing 
\begin{align}
\textbf{z}_j = \argmax_{\textbf{z}\in X} \sigma_{j-1}(\textbf{z}),\label{greedy}
\end{align}
where $\sigma_{j-1}^2(\textbf{z}) = k(\textbf{z},\textbf{z})-\textbf{k}_{Z_{1:j-1}}(\textbf{z})^TK_{Z_{1:j-1}}^{-1}\textbf{k}_{Z_{1:j-1}}(\textbf{z})$ is the \emph{conditional variance} of the GP \citep[see ][for detailed description and discussion]{hennig2016exact,burt2019rates}.
Therefore, we refer to this method of selection as conditional variance reduction (CVR), as it selects the datapoint with the highest conditional variance as the next inducing point, in hopes that this variance will be reduced.


\section{Inducing Point Allocation for Bayesian Optimisation}
\label{sec:bo}
BO typically requires updating the surrogate model(s) at each step to leverage the latest information available, so it makes sense to include updating the inducing point locations (see Algorithm \ref{alg:BO}). 
In the case of CVR, which requires a kernel, \cite{vakili2021scalable} use the kernel fitted during the previous BO step.
Unfortunately, as we will demonstrate across all our experimental results, regression-inspired IPA strategies are not satisfactory for use within BO loops. While a level of global accuracy is needed to prevent the re-investigation of areas already identified as sub-optimal, Figure \ref{fig::inducing_points} shows that accurate modelling in promising areas is necessary to allow the precise identification of the optimum. For a more formal intuition into the unsuitability of existing IPA strategies see Appendix \ref{appendix:theory}. For these reasons we now propose DPP-based IPA strategies that are able to change the relative trade-off of local and global modelling capabilities.

\begin{algorithm}
\DontPrintSemicolon
\caption{High-throughput BO with SVGPs}
\label{alg:BO}
\KwIn{Resource Budget $R$, Batch size $B$}
Initialise $n\leftarrow0$ and spent resource counter $r\leftarrow0$ \;
Collect initial design $D_0$ and fit initial model $\mathcal{M}_0$ \;
\While{$r\leq R$}{
    Begin new iteration $n\leftarrow n+1$ \;
    Build IPA $Z_n$ using $D_{n-1}$ and $\mathcal{M}_{n-1}$\;
    Fit model $\mathcal{M}_n$ using IPA $Z_n$ to data $D_{n-1}$ \;
    Generate $B$ query points $\{\textbf{x}_i\}_{i=1}^B$ \;
    Collect evaluations $D_n \leftarrow D_{n-1}\bigcup\{(\textbf{x}_i, y_{\textbf{x}_i})\}_{i=1}^B$\;
    Update spent budget $r\leftarrow r + B$
    \;
}
\KwRet{Believed optimum across $\{\textbf{x}_1,..\textbf{x}_R\}$}
\end{algorithm}

We now provide the primary contribution of this work ---  a general method for IPA suitable for down-stream decision making tasks. Unlike existing IPA strategies, our proposed methods ensure the model focuses its resources on promising (local) areas of the space whilst maintaining a sufficiently accurate global model. 

\subsection{A General IPA Formulation}

\textbf{Quality-Diversity Decomposition.}  Although CVR only leverages the repulsive properties of DPPs, it is also possible, through a convenient reparameterisation, to encode a notion of the quality of the sampled points.
Consider the DPP defined as in (\ref{eq:DPP}) but with $K_Z$ replaced  by 
\begin{align}
L_{Z}=\left[q(\textbf{z}_i)k(\textbf{z}_i,\textbf{z}_j)q(\textbf{z}_j)\right]_{(\textbf{z}_i,\textbf{z}_j)\in Z\times Z},\label{eq:qDPP}
\end{align} 
where $q:\mathcal{X}\rightarrow\mathds{R}$,
i.e., we observe $Z$ with probability
    $\mathds{P}(\mathcal{Z}=Z)\propto \big|L_Z\big|$.
In our case, we can choose $q:\mathcal{X}\rightarrow\mathds{R}^+$ so that it can be seen as a \textit{quality function}, designed to provide large values for points lying in promising areas of the space and low values elsewhere. Indeed, due to the decomposition
\begin{align}
    |L_Z|=|K_Z|*\prod_{i=1}^Nq(\textbf{z}_i)^2, \label{eq:decomp}
\end{align} 
as derived in Section 3.1 of \cite{kulesza2012determinantal}, it is clear that a particular $Z$ will occur with high probability only if it contains points that have large quality scores (as measured by $q(\textbf{z}_i)$) \textbf{and} have a diverse spread (as measured by $|K_Z|$), see Figure \ref{fig::qual}. 
Hence, this constitutes an intuitive tool for building IPAs well-suited to the demands of BO (see Figure \ref{subfig:CIR} for a demonstration, and Appendix \ref{appendix:theory} for a more formal justification).

\begin{figure}%
\subfloat{\includegraphics[height = 0.18\textwidth]{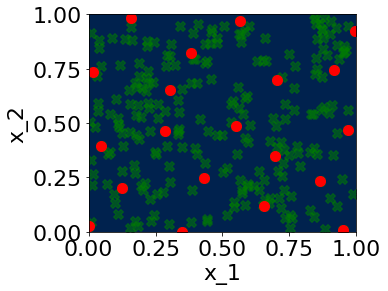}}%
\subfloat{\includegraphics[height = 0.18\textwidth]{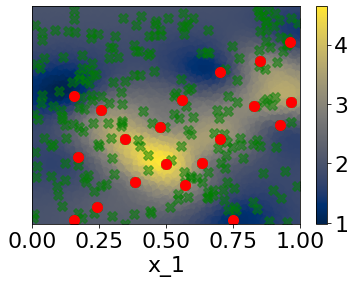}}%
\caption{25 elements (red) chosen from 250 candidates (green) by a DPP with (left)  constant and (right) locally varying quality functions (background colour).}%
\label{fig::qual}%
\end{figure}

\textbf{Greedy (Approximate) Maximisation.} Given a particular $q$, we can simply apply the same greedy algorithm used by CVR, just with $K_Z$ replaced by $L_Z$, to efficiently build a set of BO-specific inducing points as an $O(NM^2)$ approximate MAP estimate of the DPP implied by Eq.~\ref{eq:qDPP}. Conveniently, this resulting MAP estimate has an intuitive interpretation, as specified in Theorem \ref{theorem:greedy}, with proof in Appendix \ref{appendix:theorem1}. 

\begin{restatable}{theorem}{greedy}\label{theorem:greedy} Suppose inducing points $\mathcal{Z}$ are distributed according to a DPP with similarity kernel $k: \mathcal{X}\times\mathcal{X}\rightarrow\mathds{R}$ and quality function $q: \mathcal{X}\rightarrow\mathds{R}$, i.e., $\mathds{P}(\mathcal{Z}=Z)\propto \big|L_Z\big|$. Then, according to the greedy approximation, the $j^{th}$ component of the MAP estimate of $\mathcal{Z}$  is given by
\begin{align}
    \textbf{z}_j = \argmax_{\textbf{z}\in X}\: q(\textbf{z})\sigma_{j-1}(\textbf{z}),\label{eq:greedy2}
\end{align}
where $\sigma^2_{j-1}(\textbf{z})$ is the conditional variance of the noise-free GP model conditioned on the already selected points $Z_{1:j-1}$ (cf. Eq.~\ref{greedy}).
\end{restatable}

\subsection{Choosing a Quality Function}
While any quality function can be used, in practice $q$ should be carefully chosen to deliver the right quality-diversity trade-off. Intuitively from Eq.~\ref{eq:qDPP}, the relative amplitudes of the quality function and the similarity kernel are key to this trade off, so for effective IPA, $q$ must be chosen to complement $k$ (rather than dominate or be dominated by it). 

In addition, we propose the following four properties to guide our choice of the quality functions:
%
\begin{itemize}
    \item \textbf{Discriminative:} $q$ should return large values in areas of the space that are worthwhile modelling whilst providing smaller contrasting values elsewhere. This means high values in regions expected to be close to the optimum and/or with large predictive uncertainty.
    \item \textbf{Informative:}
    $q(\textbf{z})$ should encode our current knowledge about the objective function $f$ at $\textbf{z}$,
    which is available through the already-collected evaluations $y_i=f(\textbf{x}_i)$ and/or the surrogate model(s) of the previous BO step. 
    \item 
\textbf{Shift invariance:} the resulting IPA should be invariant to adding an offset to the data. Given a GP model, adding an offset should only affect its mean function and leave the kernel unchanged, which means that the quality function must also be insensitive to shift.
    \item 
\textbf{Scale invariance:} the resulting IPA should be invariant to linear re-scaling of the data. Given a GP model, a multiplicative factor on the data may result in a multiplicative factor on the kernel. Hence, for eq.~\ref{eq:greedy2} to deliver identical results, we need $q$ to be invariant to re-scaling, up to a multiplicative (positive) constant.
%
\end{itemize}



\subsection{A Linear Quality Measure}
\label{subsec:CIR}
We propose here a simple and intuitive choice for the quality function that shows strong empirical performance (see Section \ref{sec:exp}). Many other choices are possible, for example, we also derived a quality function based on information-theoretic considerations. Although well-motivated, we found this entropy-based approach to be less effective, likely due to the computational approximations required, than the simpler choice that we are about to present. To streamline our exposition, the derivation and results of the information-theoretic approach are deferred to Appendix \ref{appendix:info-theory}. 

\textbf{Noise-free evaluations.}
A natural quality function (for a single-objective maximisation problem) that satisfies the four above-mentioned properties is the following linear function of $y_i$:
\begin{align}
    q_{\textrm{Lin}}(\textbf{z}_i) = y_i - \hat{f}, \qquad \text{with } \hat{f} = \min_i y_i. \label{eq:lin}
\end{align}
A linear rescaling of the data will change $q$ by a multiplicative factor only, and subtracting $\hat{f}$ makes it positive and shift invariant. Furthermore, $\hat{f}$ ensures the discriminative property, i.e. that $q$ is zero at the worst observation and largest at the best. 

\textbf{Noisy evaluations.} For problems with large observation noise, $y_i$ can give misleading estimates of $f(\textbf{z}_i)$ and so it is unwise to use (\ref{eq:lin}). However, in these settings, we can make use of the previous BO step's surrogate model $\mathcal{M}_{n-1}$ and instead calculate the expected value of $q_{\textrm{Lin}}$. Additionally, to ensure positivity, we swap the linear function for the (piece-wise linear) Recitified Linear Unit (ReLU), yielding the quality function
\begin{align}
    q_{\textrm{IMP}}(\textbf{z}_i) = \mathds{E}_{f \sim \mathcal{M}_{n-1}}\left[\max(f(\textbf{z}_i)-\hat{f},0)\right],
    \label{eq:relu}
\end{align}
where the baseline is now the minimal predicted value of the objective function, i.e. $\hat{f} = \min_{\textbf{x}\in D_n} \mu_{n-1}(\textbf{x})$ for $\mu_{n-1}(\cdot)$ the posterior mean of $\mathcal{M}_{n-1}$. Note that (\ref{eq:relu}) takes the form of the well-known Expected Improvement (EI), just with a modified baseline, and so can be calculated in closed form \citep[see][]{jones1998efficient}.

Reassuringly, the performance of (\ref{eq:relu}) is robust to the specific choice of baseline, with Appendix \ref{appendix:ablation} showing negligible performance differences when using the minima or mean of the predicted objective function values, or even when using a softplus relaxation of the ReLU. However, significantly tightening the baseline to be the maximum of the objective function (as typically used by EI acquisition function) yields a dramatic drop in performance. Indeed, EI is not designed to discern between all our collected points, only to help identify where there could be new maxima.

\subsection{Beyond Single Objective BO}
The quality function described above is designed for single-objective BO problems. However, our approach based on the quality-diversity decomposition of DPPs is a general way to ensure that sparse models are accurate in the areas where they will be used and so may apply to a much larger variety of optimisation problems, 
including those with constraints, multiple objectives, and more generally active learning problems such as level set estimation. The quality function should be tailored to each problem: for instance in level-set estimation, the important regions to model are not regions where the output value is maximal, but where it is close to the targeted level. In Section \ref{sec:exp} we demonstrate such extensions.

\section{Related Work}
\label{sec:related}

\textbf{Alternative Sparse Surrogate Models.} Three other formulations of sparse GPs have been used in BO loops. Firstly, \cite{mcintire2016sparse} propose a compelling modification of sparse online GPs\citep{csato2002sparse}, where they up-weight promising areas of the feature space (as measured by the expected improvement of candidate evaluations). However, online GPs, which see only a single pass of the data, provide worse approximations than SVGPs, which have multiple chances to learn from each datapoint. Moreover, due to its requirement of $N$ individual challenging optimisations for each individual model fit, this approach is unsuitable for the high-throughput scenarios tackled in this paper (and consequently was only tested by \cite{mcintire2016sparse} on problems with $M=30$ and $N=60$). Another way to alleviate the cost of GP inference is by  approximating the spectral density of its kernel \citep{lazaro2010sparse}. However, spectral approximations are not appropriate for BO as they seek to preserve global structure and, as such, have no way of providing local high-fidelity modelling. Indeed, applying spectral GPs to BO requires expensive and heuristic modifications to its loss function \citep{yang2021sparse} and even then fails to match the performance of exact GPs. In contrast to these two alternatives, our proposed approach retains the state-of-the-art computational complexity and performance of SVGPs. Finally, \cite{maddox2021conditioning} (with similar work by \cite{chang2022fantasizing}) propose the OVC method for the fast conditioning of SVGPS, allowing efficient calculation of popular look-ahead acquisition functions, albeit those outside of the high-throughput domain. 

\textbf{Additional uses of DPPs in BO.} Outside of IPA, DPPs are also commonly used in the context of batch BO, where the goal is to recommend diverse collections of points. 
Prominent examples include the approaches of \cite{kathuria2016batched, dodge2017open} and \cite{nava2022diversified}, as well as \cite{moss2021gibbon} where, similarly to our ENT-DPP, an information-theoretic motivation is used to inform the construction of the DPP's diversity and quality terms. DPPs have also been used in high-dimensional BO \citep{wang2017batched} to sample diverse subsets of the available search space dimensions.

\textbf{Scalable BO via Local Models.} A popular alternative approach for BO under large evaluation budgets is to use multiple cheaper local models in lieu of a single expensive global model \citep{gramacy2015local, rulliere2018nested, cole2021locally, cole2022large}. Particularly powerful BO routines employing local models like TURBO of \citep{eriksson2019scalable} and, for multi-objective BO, MORBO \citep{daulton2022multi} 
are ideal for applications where the only goal is to find a reasonable solution because global modelling (and optimisation) is challenging (e.g., for high-dimensional optimisation problems). However, the local models built by TURBO are not always useful in settings where the goal is to collect data that allows the building of a useful final model, e.g.,  in the active learning applications we consider below or when we need a rough understanding of global behaviour to ensure global convergence. 

\section{Experimental Results}
\label{sec:exp}

We now provide an empirical evaluation of our proposed IPA framework across a suite of high-throughput BO problems using the open-source BO library \textsc{Trieste} \citep{Berkeley_Trieste_2022}. We then illustrate the general applicability of our IPA framework, by demonstrating how quality functions can be designed for multi-objective and active learning problems. Additional experimental details are contained in our appendices. Implementations of our IPAs are contained within the \textsc{Trieste} \citep{Berkeley_Trieste_2022} and  \textsc{BoTorch} \citep{balandat2020botorch} libraries.

\subsection{ Single Objective Optimisation}
For clarity, all our synthetic experiments follow the same setup. We consider an SVGP model with either $M=250$ or $500$ inducing points using either 1) our proposed IPA strategy with the improvement-based quality function (\ref{eq:relu}) which we call IMP-DPP , 2) the CVR of \cite{burt2019rates} (see Section \ref{sec:regression}), 3) choosing the centroids of a K-means clustering of the data, and 4) choosing inducing points spread uniformly across the search space.  SVGP models are fit using an Adam optimiser with learning rate $0.1$, using an early stopping criteria with a patience of $50$ and a learning rate halving on plateau schedule with a patience of $10$.

For all DPP-based IPAs we follow the thoroughly tested approach of \cite{burt2020convergence} and \cite{vakili2021scalable} and use the kernel of the previous BO step's model to allocate the IPA and then refit the kernel when training the current BO step's model on the new data and the chosen IPA. We allocate a total evaluation budget of $N=5{,}000$ evaluations split across $50$ BO steps in batches of $100$ points. When the total number of queried points is less than the desired number of inducing points (e.g. for the first 4 optimisation steps when $M=500$), we use just the $N$ available training points as our IPA.
 Subsequent batches are collected using the decoupled Thompson sampling scheme presented in \cite{vakili2021scalable}. 
We use $100$ random Fourier features to build a Fourier representation of samples and maximise each using an L-BFGS optimiser starting from the best of a random sample of $10{,}000$ points. BO using an exact GP model is included as a baseline; however, we can report only the first $10$ optimisation steps, after which it became prohibitively expensive (i.e., for $N>1{,}000$).

Figure \ref{fig::results} demonstrates optimisation performance across the 4d Shekel, 5d Michalewicz, 5d Ackley, 6d Hartmann, and 4d Rosenbrock functions (see Appendix \ref{appendix:experiments} for definitions), where we have contaminated the evaluations of each with Gaussian noise of variance $0.01$, except for the easier Hartmann where we consider a larger variance of $0.1$. Note that we re-scaled these baselines so that they have a variance of $1.0$ (under random samples across the search space) and so these noise levels are large, resulting in challenging optimisation tasks. Unsurprisingly, greater performance is achieved when using larger number of inducing points for all the considered methods, except for the easier Rosenbrock function, where all methods perform equally well. For the Shekel and Michalewicz functions, only IMP-DPP achieves precise optimisation, even when using just $M=250$ inducing points. For the Michalewicz function, IMP-DPP with $M=500$ provides a dramatic improvement over the other methods. In contrast, on the Hartmann function we see all $M=500$ approaches, as well as IMP-DPP with $M=250$, achieve comparable optimisation performance. In addition to performing improved optimisation, we show (in Appendix \ref{wallclock}) that our SVGP-based approaches incur significantly lower computational overheads than exact GPs. Moreover, unlike the exact GP, the SVGP approaches maintain a constant overhead as BO progresses. 

Interestingly, for some of the more challenging functions considered in Figure \ref{fig::results}, the exact GP leads to optimisation that gets stuck in local minima, whereas the SVGP approaches are able to fully converge. We hypothesise that SVGPs have an advantage in these non-stationary settings as they are able to ignore promising yet not optimal areas of the space that would otherwise mislead the algorithm ---  a helpful consequence of their limited modelling resources. Similar behaviour is noted by \cite{maddox2021conditioning}
 when also using SVGPs for online decision making.

\begin{figure}%
\subfloat[4d Shekel Function]{\includegraphics[height = 0.20\textwidth]{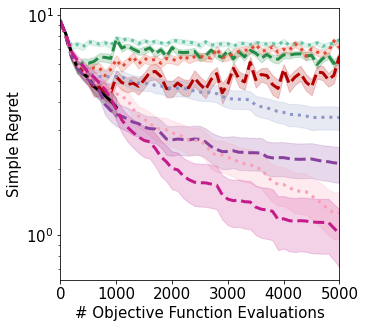}}%
\subfloat[5d Michalewicz Function]{\includegraphics[height = 0.20\textwidth]{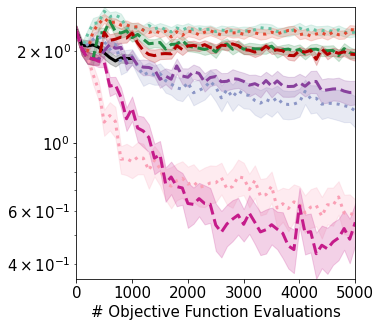}}%
\\
\subfloat[5d Ackley Function]{\includegraphics[height = 0.20\textwidth]{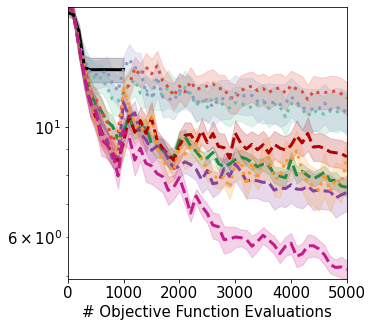}}%
\subfloat[6d Hartmann Function]{\includegraphics[height = 0.20\textwidth]{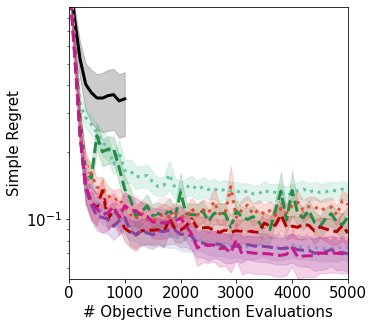}}%
\\
\subfloat[4d Rosenbrock Function]{\includegraphics[height = 0.20\textwidth]{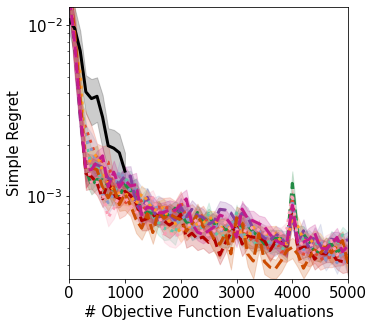}}%
\subfloat{{%
\hspace{-0.7cm}
\setlength{\fboxsep}{5pt}%
\setlength{\fboxrule}{0pt}
\fbox{\includegraphics[height = 0.18\textwidth]{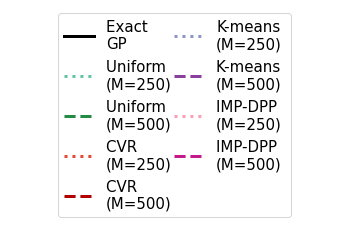}}%
}}%
\caption{ Results are averaged over $50$ runs and we report the mean and its $95\%$ confidence intervals for the simple regret of the maximiser of the
posterior mean across previously queried points. Our proposed IMP-DPP is the only IPA strategy that provides consistently high performance.
}%
\label{fig::results}%
\end{figure}

\subsection{Active Learning}
\begin{figure*}%
\subfloat[Ground truth of test set]{\includegraphics[height= 0.23\textwidth]{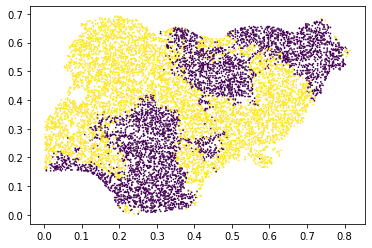}}%
\subfloat[SVGP using CVR ]{\includegraphics[height= 0.23\textwidth]{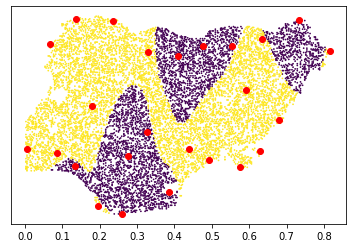}}%
\subfloat[SVGP using our active learning IPA]{\includegraphics[height= 0.23\textwidth]{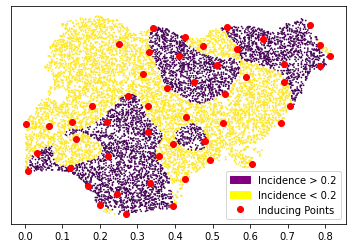}}%
\caption{ Incidence threshold breaches predicted by two surrogate models each fine-tuned over $10$ steps of high-throughput active learning. (a) Performance is evaluated across a randomly sampled held-out test set. (b) CVR fails to accurately learn the complex classification boundary and obtains an accuracy of only $79\%$. (c) In contrast, our proposed IPA focuses inducing points (red dots) along the classification boundary, yielding a improved model with an accuracy of $89\%$.
}%
\label{fig::malaria}%
\end{figure*}

To demonstrate the generality of our proposed IPA framework, we now depart from single-objective BO and instead consider an active learning task inspired by \cite{balandat2020botorch}. We wish to learn which spatial locations in Nigeria have rates of a malaria-causing parasite \textit{Plasmodium Falciparum} over a critical threshold.

We model the occurrence of a breach in the critical threshold at location $\textbf{x}$ through a Bernoulli likelihood $y_{\textbf{x}}|f\sim \mathcal{B}(\Phi(f(\textbf{x})))$ where $f$ denotes a latent sparse GP with $50$ inducing points and $\Phi:\mathds{R}\rightarrow[0,1]$ is the inverse probit function \citep[see][for details]{hensman2015scalable}. Starting from a random initial design of 100 evaluations, we then use the BALD acquisition function of \cite{houlsby2011bayesian} to sequentially improve our model over $10$ data acquisition steps, each time collecting evaluations at $100$ informative locations then updating the classification surrogate models. 

As the performance of the classifier is determined by the accuracy of its classification boundary (i.e., where  $f \approx 0$), it is natural to consider a quality function that encourages the placing of inducing points where $|f|$ is small. To this end, we consider the active learning quality function
\begin{align}
     q_{\textrm{AL}}(\textbf{z}) = \mathds{E}_f\left[ \hat{f} - |f(\textbf{z})|\right], \label{eq:qualityAL}
\end{align}
where $\hat{f}=\max(|\max(f)|, |\min(f)|)$ is the largest absolute value obtained by the latent GP. This quality function has maximal score at $f=0$, i.e., the level set of the latent GP corresponding to the classification boundary. 
Figure \ref{fig::malaria} demonstrates the benefit of using this custom quality function to drive IPA in the considered active learning problem.

 \subsection{Multi-objective Optimisation}

In multi-objective optimisation (MOO) we seek to find high-performing solutions according to $K$ ($\geq 2$) competing objective functions $f^1(\textbf{x}),\dots,f^K(\textbf{x})$. In these tasks, where improvements in one objective may harm another, the ability to characterise trade-offs between these competing objectives becomes crucial. Consequently, multi-objective optimisation corresponds to finding the so-called Pareto set  which contains all locations representing optimal trade-offs, i.e., those that cannot be perturbed to yield an improved score in a single objective without a deterioration in the score of another objective (see \cite{emmerich2005single} for an introduction). Therefore, when using sparse models as surrogate models for  MOO,  it is no longer sufficient to focus modelling resources into the "best" areas of the space; rather, we want to focus coverage around the Pareto front. Therefore we consider the quality function
\begin{align}
    q_{\textrm{HV}}(\textbf{z}) = \mathds{E}_{f_1,\dots,f_K}\left[\prod_{k=1}^K\max(f_k(\textbf{z}_i)-\hat{f}_k,0)\right],
    \label{eq:qual_moo}
\end{align}
where $f_k$ represents the model of the $k^{th}$ objective function and $\hat{f}_k$ its minimal value, i.e., we consider a product of our single objective quality functions. As Eq. \ref{eq:qual_moo} can be interpreted as the Hyper-Volume (HV) of the set containing all the previously collected points that are dominated by $\textbf{z}$, we refer to the IPA resulting from this quality function as HV-DPP. We allocate inducing points for each model separately but use the same shared quality function (that uses information from all the models) to encourage the allocation of points along the Pareto front. Although $q_{\textrm{HV}}$ has a strong bias for points in the central area of the front, it is fast to evaluate and we found it adequate for enabling effective high-throughput BO. Future work will build a more sophisticated quality function that provides an even focus along the whole Pareto front.

\textbf{Synthetic benchmark.} Figure \ref{fig::moo_results} demonstrates high-throughput optimisation of a noisy variant of the 4-dimensional ZDT3 problem (see Appendix \ref{appendix:experiments} for a problem description). We start with 100 random evaluations and use the Chebyshev scalarisation acquisition function described by \cite{paria2020flexible} to collect 50 batches of 100 evaluations for the sparse methods (each with 100 inducing points), and 10 batches for the exact GP. As is standard practice in multi-objective optimisation, we measure performance in terms of the difference between the hyper-volume dominated by the true Pareto optimal front and the one found by BO.

\begin{figure}%
\subfloat[4d ZDT3]{\includegraphics[height= 0.53\columnwidth]{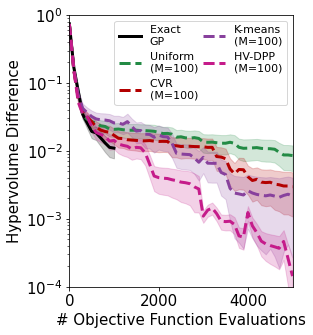}}%
\subfloat[9d Heat Exchanger Design]{\includegraphics[height= 0.52\columnwidth]{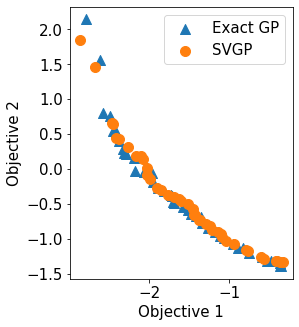}}%
\caption{ Demonstration of high-throughput multi-objective optimisation. (a) Optimisation of the 4-dimensional ZDT3 problem over 50 runs, where only our proposed HV-DPP matches and then builds upon the performance of the exact GP. (b) The Pareto fronts found for the challenging heat exchanger design task by 1) an SVGP with our IPA strategy on only $100$ inducing points alongside, and 2) exact GPs.
}%
\label{fig::moo_results}%
\end{figure}

\textbf{Real-world problem.} For our final example, we turn to the problem of designing an effective yet light-weight automotive heat exchanger (radiator), as considered by \cite{paleyes2022penalisation} (see Appendix \ref{appendix:experiments} for a full description). This challenging 9d problem has two objectives and three constraints, so requires five surrogate models, each of which will need IPA. For the constraint models, we use the $q_{AL}$ quality function (as presented for the active learning task) and for the objective models we use the $q_{HV}$ quality function. We start with $100$ random evaluations and use \cite{paleyes2022penalisation}'s HIPPO acquisition function to allocate $10$ batches of $100$ further evaluations. Figure \ref{fig::moo_results} shows that an SVGP surrogate model using our proposed IPA and only $100$ inducing points is able to find a comparable Pareto front to an expensive exact GP. Moreover, in Appendix \ref{wallclock}, we provide wall-clock timing for these experiments, demonstrating that the SVGP incurs order-of-magnitude lower optimisation overheads than the exact GP.


%

%


\section{Conclusions and Further Work}
We have proposed the first BO-specific methods for selecting the locations of inducing points in sparse GPs. By exploiting the quality-diversity decomposition of DPPs, we are able to dramatically improve DPP-based IPA, transforming what is often, in the context of BO, a poorly performing IPA \citep[the conditional variance reduction of ][]{burt2019rates} to the best (our IMP-DPP). Moreover, we have shown that our proposed framework provides a general framework for ensuring that sparse GPs are accurate in key areas, and so has applications across a range of down-stream tasks.

In future work we will apply our BO-specific IPAs to real-world problems where sparse GPs are already being used, e.g., quantile optimisation \citep{torossian2020bayesian}. We will also investigate their applicability to other inducing point-based methods that are also used in decision making loops, like deep GPs \citep{damianou2013deep}. Moreover, our SVGPs could also be applied to high-dimensional optimisation problems by extending single-model trust region approaches \citep{diouane2022trego} to support large optimisation budgets. Finally, note that our proposed IPA does not require Euclidean input spaces (unlike standard SVGP formulations which optimise inducing point locations using gradient descent). Therefore, we also wish to use our proposed scalable surrogate models to enable high-throughput versions of active learning loops over discrete structures that can be modelled with GPs, e.g., genes \citep{moss2020boss} and molecules \citep{moss2020gaussian,thawani2020photoswitch,griffiths2022gauche,rankovic2022bayesian}.



\bibliography{refs}

\appendix
\onecolumn

\section{Theoretical Justification}
\label{appendix:theory}
For a more formal intuition into the unsuitability of existing IPA strategies for BO and the suitability of our proposed method we can apply Lemma 4 of \cite{burt2020convergence}, which bounds the approximation error between an exact GP and its sparse approximation and is restated below for convenience.
\begin{theorem}\textbf{(Burt et al., 2020, Lemma 4)}
\label{theorem:burt_orig}
Suppose $\textbf{y}|X,Z\sim\mathcal{N}(0,K_{X}+\sigma^2I_N)$. Then, for our exact posterior $P$ and our variational posterior $Q$, we have that for any $X$ and $Z$
\begin{align*}
    t(Z)/(2\sigma^2)\leq \mathds{E}\left[\textrm{KL}[Q||P]|Z,X\right]\leq t(Z)/\sigma^2,
\end{align*}
where $KL[Q||P]$ denotes the Kullback-Leibler divergence and $t(Z) = tr(K_{X} - Q_{X}(Z))$,  for  $Q_{X}(Z):= K_{Z}(X)^TK_{Z}^{-1}K_{Z}(X)$. 

\end{theorem}

Now, we can state our Corollary \ref{theorem:burt} which bounds the maximal and minimal discrepancy between an SVGP and an exact GP over a localised area.
\begin{corollary}\label{theorem:burt}
Suppose $\textbf{y}|\textbf{X},\textbf{Z}\sim\mathcal{N}(0,K_{X}+\sigma^2I_N)$. Then, for our exact posterior $P$ and our variational posterior $Q$, we have that for any subspace $A\subseteq\mathcal{X}$
\begin{align*}
    t_A(Z)/(2\sigma^2)\leq \mathds{E}\left[\textrm{KL}_A[Q||P]|Z,X\right]\leq t_A(Z)/\sigma^2.
\end{align*}
Here $KL_A[Q||P]$ denotes the Kullback-Leibler divergence calculated only over datapoints contained in the subspace $A\subseteq\mathcal{X}$ and $t_A(Z) = tr(K_{X_A} - Q_{X_A}(Z))$,  where $X_A$ denotes the subset of data $X$ lying in $A$ and $Q_{X_A}(Z):= K_{Z}(X_A)^TK_{Z}^{-1}K_{Z}(X_A)$. 
\end{corollary}
\begin{proof}
Noting that
\begin{align}
    KL_A[Q||P]=KL\left[\mathcal{N}(0, K_{X_A}+\sigma^2 I)||\mathcal{N}(0, Q_{X_A} + \sigma^2I)\right] \nonumber,
\end{align}
we can directly apply Lemma 4 of  \cite{burt2020convergence} but with $K_{X_A}$ instead of $K_{X}$ and $Q_{X_A}$ instead of $Q_{X}$ to get the required result.
\end{proof}

By applying Corollary \ref{theorem:burt} in the case $A=\mathcal{X}$, \cite{burt2019rates} justify that the goal of IPA strategies should be to minimise the trace term $t_{\mathcal{X}}(Z)$, as minimising $t_{\mathcal{X}}(Z)$ ensures a small divergence between our SVGP and the exact GP \citep[a notion formalised in][]{matthews2016sparse}, which in turn \citep[through Proposition 1 of][]{burt2019rates} ensures a small approximation error across the whole search space. Sampling inducing points from a DPP can then be justified as, under such a sampling scheme, the resulting  $t_{\mathcal{X}}(Z)$ lying within a constant factor of its minimal achievable value \citep[see Theorem 1 of][]{belabbas2009spectral}.

Now suppose that, as is the case in BO (see Figure \ref{fig::inducing_points}), we want to ensure our variational approximation is especially accurate in a particular region $A\subseteq \mathcal{X}$, i.e., we seek IPAs with small values of $t_{A}(Z)$. However, noting the decomposition $t_{\mathcal{X}}(Z)=t_{A}(Z)+t_{\mathcal{X}/A}(Z)$, it is clear that the minimisation of $t_A(Z)$ forms only  a small part of the overall IPA objective targeted by DPP sampling, especially in BO applications where the promising areas $A$ often contains only a small proportion of the search space.

We can now understand the benefit of including a quality function in our DPP through Corollary \ref{theorem:quality}, an extension of Lemma 1 of \cite{belabbas2009spectral}. Before presenting our Corollary, we restate this Lemma as it is used by \cite{burt2020convergence} to justify using a DPP for IPA. 

\begin{theorem}\textbf{(Belabbas and Wolfe, 2009, Theorem 1)}
\label{theorem:belabbas}
Let  $\eta_1\geq ... \geq \eta_N \geq 0 $ be the eigenvalues of the SPSD matrix $K_X$. Suppose a set of points $Z$ are sampled according to an $M$-determinantal point process with kernel matrix $K$. Define the matrix $Q_X(Z) = K_Z(X)^TK_Z^{-1}K_Z(X)$ and the trace term $t(Z) = tr(K_X - Q_X(Z))$. Then, 
\begin{align}
    \mathds{E}\left[ t(Z))\right]\leq (M + 1) \sum_{m=M+1}^N \eta_m.
\end{align}
\end{theorem}

Theorem \ref{theorem:belabbas} tells us that using an M-DPP to perfom IPA will make $t(Z)$ relatively close to its minimal value of $\sum_{m=M+1}^N \eta_m$ \citep[see Section 4.2.1 of][for a discussion of why $\min\limits_{Z\in\mathcal{X}^M}t(Z) = \sum_{m=M+1}^N \eta_m$]{burt2020convergence}. Therefore, in the context of Theorem \ref{theorem:burt_orig}, Theorem \ref{theorem:belabbas} guarantees low KL divergence between an exact GP and the sparse approximation arising from a DPP IPA strategy.

Now we can present our Corollary \ref{theorem:quality}, which adapts the above results for M-DPPs with quality functions (i.e., $K\rightarrow L$). In particular, we show that when sampling from a DPP with a quality function $q$, we are guaranteed (in expectation) to have $t^q(Z)$ lying within a constant factor of its minimal achievable value. Therefore, by increasing $q$ in our promising region $A$ and reducing it elsewhere, we increase the contribution of the individual components corresponding to the terms of $t_A(Z)$ in the overall objective  $t^q(Z)$, thus ensuring low KL divergence between our SVGP and the exact posterior in $A$.

\begin{corollary}\label{theorem:quality}
Suppose that $Z$ is sampled from a DPP with quality function $q: \mathcal{X}\rightarrow\mathds{R}^+$. Then
\begin{align*}
    \mathds{E}\left[ t^q(Z) \right] \leq (M+1) \hat{t}^q
\end{align*}
for a $q$-weighted trace term
$t^q(Z) = \sum_{i=1}^N q(\textbf{x}_i)^2\left(\left[K_{X}\right]_{i,i} - \left[Q_X(Z)\right]_{i,i}\right)$ and its minimal achievable value $\hat{t}^q =  \min\limits_{Z'\in\mathcal{X}^M} t^q(Z')$.
\end{corollary}
\begin{proof}

A direct application of Theorem \ref{theorem:belabbas} but with $K_X$ replaced by 
\begin{align*}
    L_X = \left[q(\textbf{x}_i)k(\textbf{x}_i, \textbf{x}_j)q(\textbf{x}_j)\right]_{(\textbf{x}_i,\textbf{x}_j)\in X\times X}
\end{align*}
 yields 
\begin{align*}
    \mathds{E}\left[ t_*(Z) \right] \leq (M+1) \min\limits_{Z\in\mathcal{X}^M} \hat{t}_*(Z),
\end{align*}
for $t_*(Z) = tr(L_{X} - L_Z(X)^TL_Z^{-1}L_Z(X))$. All that remains is to show that $t^q(Z) = t_*(Z)$.

Note that we can write $L_Z(X) = D_q(Z) K_Z(X) D_q(X)$, where $D_q(X)$ is a diagonal matrix with non-zero entries given by the vector $\left[q(\textbf{x})\right]_{\textbf{x}\in \mathcal{X}}$. Therefore, after routine algebraic manipulations, we have

\begin{align*}
    t_*(Z) &= tr(L_X - L_Z(X)^TL_Z^{-1}L_Z(X))\\
   &=tr(D_q(X)K_XD_q(X) - D_q(X)Q_X(Z)D_q(X))\\
   &=  \sum_{i=1}^N q(\textbf{x}_i)^2\left(\left[K_{X}\right]_{i,i} - \left[Q_X(Z)\right]_{i,i}\right) \\
   &= t_q(Z)
\end{align*}

\end{proof}

To help explain why Corollary \ref{theorem:quality} justifies the use of DPPs with quality functions as IPA strategies in BO, consider the simple demonstrative binary quality function  
\begin{align*}
    q(\textbf{z})=
\begin{cases}
\sqrt{\beta} \: &\textrm{for} \: \textbf{z}\in A,\\
\sqrt{1-\beta} \: & \textrm{otherwise},
\end{cases}
\end{align*}
under which  $t^q(Z) = \beta^2 * t_A(Z) + (1-\beta)^2*t_{\mathcal{X}/A}(Z)$. Therefore, by varying the quality function (through $\beta$), we re-weight the contributions $t_A(Z)$ and $t_{\mathcal{X}/A}(Z)$ in $t^q(Z)$, i.e., we change the relative trade-off allocated by our DPP on local (inside $A$) and global (outside $A$) modelling.  Although our practical recommendations for quality functions in Section \ref{sec:bo} are much more sophisticated than this binary function, the intuition remains the same.

\section{Proof of Theorem \ref{theorem:greedy}}
\label{appendix:theorem1}

We now restate and prove Theorem \ref{theorem:greedy} that demonstrates the effect of the choice of $q$ on the inducing points chosen by the IPA. This Theorem is used in the main paper to explain the constraints imposed upon $q$ in order to achieve scale and translation invariant IPA strategies.

\greedy*


\begin{proof}

As derived in Eq. (4) of \cite{hennig2016exact}, greedy maximisation of a DPP with a kernel $k$ corresponds to setting the $j^{th}$ component of the MAP estimate of as
\begin{align}
    \textbf{z}_j = \argmax_{\textbf{z}\in X}\: \sigma_{j-1}(\textbf{z}),\nonumber
\end{align}
where $\sigma_{j-1}^2(\textbf{z}) = k(\textbf{z},\textbf{z})-\textbf{k}_{Z_{1:j-1}}(\textbf{z})^TK_{Z_{1:j-1}}^{-1}\textbf{k}_{Z_{1:j-1}}(\textbf{z})$.

We can apply exactly the same derivation to a DPP with a  similarity kernel $l(\textbf{x},\textbf{x}')= q(\textbf{x})k(\textbf{x},\textbf{x}')q(\textbf{x}')$, i.e., as arising from including a quality function $q$, to get

\begin{align}
    \textbf{z}_j = \argmax_{\textbf{z}\in X}\: \hat{\sigma}_{j-1}(\textbf{z}), \label{eq:sigma}
\end{align}
where  $\hat{\sigma}_{j-1}^2(\textbf{z}) = l(\textbf{z},\textbf{z})-\textbf{l}_{Z_{1:j-1}}(\textbf{z})^TL_{Z_{1:j-1}}^{-1}\textbf{l}_{Z_{1:j-1}}(\textbf{z})$ for $\textbf{l}_{Z_{1:j-1}}(\textbf{z}) = \left[l(\textbf{z}', \textbf{z})\right]_{\textbf{z}' \in \mathcal{Z}_{1:j-1}}$ and $L_{Z_{1:j-1}} = \left[l(\textbf{z},\textbf{z}')\right]_{\textbf{z},\textbf{z}' \in Z_{1:j-1} \times Z_{1:j-1}}$.

Note that we can expand $\textbf{l}_{Z_{1:j-1}}(\textbf{z}) = D_q(Z_{1:j-1})\textbf{k}_{Z_{1:j-1}}(\textbf{z})q(\textbf{z})$ and $L_{Z_{1:j-1}} = D_q(Z_{1:j-1}) K_{Z_{1:j-1}}D_q(Z_{1:j-1})$ for a diagonal matrix $D_q(Z_{1:j-1})$ with non-zero entries given by the vector $[q(\textbf{z}_i)]_{i=1}^{j-1}$. Therefore, we can expand Equation \ref{eq:sigma} as 

\begin{align*}
    \hat{\sigma}_{j-1}^2(\textbf{z}) = &  q(\textbf{z})k(\textbf{z},\textbf{z})q(\textbf{z}) \\
    &- q(\textbf{z})\textbf{k}_{Z_{1:j-1}}^T(\textbf{z})D_q(Z_{1:j-1})\\&*D_q(Z_{1:j-1})^{-1}K_{Z_{1:j-1}}^{-1}D_q(Z_{1:j-1})^{-1}\\&*D_q(Z_{1:j-1})\textbf{k}_{Z_{1:j-1}}(\textbf{z})D_q(Z_{1:j-1})q(\textbf{z}) \\
    = & q(\textbf{z})^2 \sigma_{j-1}^2(\textbf{z}).
\end{align*}

Therefore, for the greedy MAP estimate of a DPP with quality function $q$, we have
\begin{align}
    \textbf{z}_j &= \argmax_{\textbf{z}\in X}\: \hat{\sigma}_{j-1}(\textbf{z}) \\\nonumber
    &= \argmax_{\textbf{z}\in X}\:q(\textbf{z})\sigma_{j-1}(\textbf{z}).\nonumber
\end{align}

\end{proof}

\section{An Information-theoretic Approach}
\label{appendix:info-theory}

We now propose our second quality function, this time chosen such that the resulting IPA, which we name ENT-DPP, can be viewed as maximising an intuitive information-theoretic quantity well-aligned with the goal of BO. This work is also available in the non-archival technical report \citet{moss2022information}.

\textbf{Reducing global uncertainty}. For GP regression tasks it is sufficient to choose points that provide maximal information about the whole objective function $f$. In fact, the arguments of \cite{srinivas2009gaussian} and \cite{hennig2016exact} demonstrate that providing the maximal reduction in uncertainty corresponds exactly to the DPP MAP objective (\ref{eq:DPP}) targeted by the CVR IPA strategy. We repeat this result below as Theorem \ref{theorem:ED}. See  \cite{srinivas2009gaussian} for a discussion of information theory in the context of GP learning and \cite{cover1999elements} for a general introduction.
\begin{theorem}\label{theorem:ED} Under a GP prior $f\sim\mathcal{GP}(0,k)$, maximising the M-DPP objective (\ref{eq:DPP}) is equivalent to choosing our inducing points $Z$ as the datapoints with evaluations $\textbf{y}_Z$ that provide the largest gain in information about the unknown function $f$
\begin{align*}
    Z=\argmax_{Z'\subseteq X : |Z'|=M} \textrm{IG}(\textbf{y}_{Z'};f).
\end{align*}
where $\textrm{IG}(\textbf{y}_Z,f) = H(f)-H(f|\textbf{y}_Z)$ quantifies the reduction in the differential entropy $H$ of $f$ provided by revealing the evaluations $\textbf{y}_Z$.
\end{theorem}

\textbf{Reducing Uncertainty in $f^*$.} In BO it is natural to also consider how much information is provided about the function's maximum value $f^* = \max f(\textbf{x})$. Taking inspiration from the empirical success of max-value entropy search acquisition functions \citep{wang2017max, takeno2019multi, moss2020mumbo,moss2020bosh,qing2022text}, where query points in BO loops are chosen to reduce our current uncertainty in $f^*$, we thus propose choosing inducing points $Z$ that maximise the following trade-off criterion:
\begin{align}
      C_{\alpha}(Z) = \alpha\times\textrm{IG}(\textbf{y}_{Z};f^*) + (1-\alpha)\times\textrm{IG}(\textbf{y}_{Z};f).
    \label{eq:MES}
\end{align}
where the information gain $\textrm{IG}(\textbf{y}_Z,f^*) = H(f^*)-H(f^*|\textbf{y}_Z)$ quantifies the reduction in uncertainty provided by these inducing points about the maximum $f^*$ \citep[see][for a derivation]{wang2017max}, and $\textrm{IG}(\textbf{y}_Z;f)$ is as in Theorem \ref{theorem:ED}. $\alpha \in [0,1]$ controls the balance of modelling resources spent on global variations and those spent in areas around potential maxima. Note that setting $\alpha=1$ returns the criterion targeted by CVR.

\textbf{Approximation with a DPP.} The lack of closed-form expression for the distribution of $f^*$ renders the calculation of this criterion (\ref{eq:MES}) challenging, and it is unclear how this objective fits into our DPP framework. Fortunately, there exists a rich literature of methods for approximately optimising similar information-theoretic quantities \citep{hennig2012entropy,hernandez2014predictive}. In particular, 
we can follow the ideas of \cite{moss2021gibbon} and use common information-theoretic inequalities to replace our desired criterion $C_{\alpha}(Z)$ with a simpler lower bound that takes the form of a DPP. Specifically, we use the well-known inequality $H(\textbf{y}_{Z}|f^*)\leq \sum_{i=1}^MH(y_{\textbf{z}_i}|f^*)$.
After simple mathematical manipulations, the resulting lower bound for the trade-off in entropy can be expressed in the following form:
\begin{align}
    C_{\alpha}(Z) \geq \frac{1-\alpha}{2}\log \big| L_{Z}\big|, \label{lowerbound} 
\end{align}
where $L_{Z}$ is defined as in (\ref{eq:qDPP}) but with quality function
\begin{align}
    q(\textbf{z}) = \exp\left(\frac{1}{M}\frac{\alpha}{1-\alpha}\textrm{IG}(y_{\textbf{z}};f^*)\right), 
    \label{quality}
\end{align}
 a function increasing in $\textrm{IG}(y_{\textbf{z}};f^*)$ --- the reduction in uncertainty provided about $f^*$ by a single observation $y_{\textbf{z}}$. 
 Although we cannot calculate $\textrm{IG}(y_{\textbf{z}};f^*)$ in (\ref{quality}) exactly, we can follow \cite{ru2018fast} and approximate $q(\textbf{z})$ with moment-matching. 

\begin{figure}%
\centering
\subfloat[4d Shekel Function]{\includegraphics[height = 0.27\textwidth]{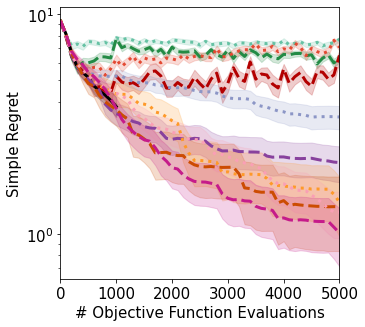}}%
\subfloat[5d Michalewicz Function]{\includegraphics[height = 0.27\textwidth]{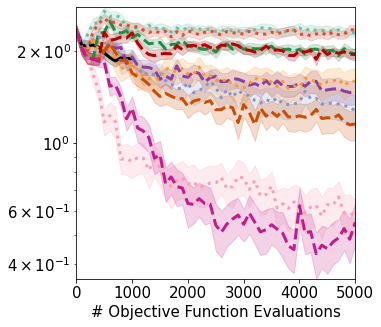}}%
\subfloat[5d Ackley Function]{\includegraphics[height = 0.27\textwidth]{ackley_final.png}}%
\\
\subfloat[6d Hartmann Function]{\includegraphics[height = 0.27\textwidth]{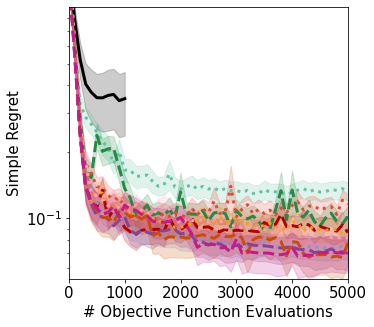}}%
\subfloat[4d Rosenbrock Function]{\includegraphics[height = 0.27\textwidth]{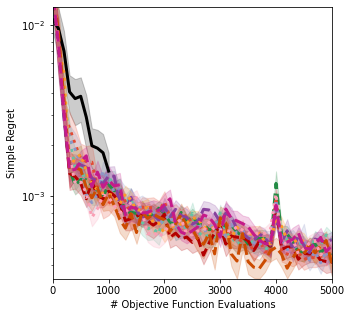}}%
\subfloat{{%
\hspace{-0.7cm}
\setlength{\fboxsep}{5pt}%
\setlength{\fboxrule}{0pt}
\fbox{\includegraphics[height = 0.25\textwidth]{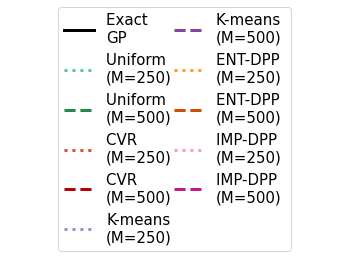}}%
}}%
\caption{ Results are averaged over $50$ runs and we report the mean and its $95\%$ confidence intervals for the simple regret of the maximiser of the
posterior mean across previously queried points. Although ENT-DPP does improve over existing approaches across most of the test functions, our IMP-DPP achieves much more impressive and consistent performance.
}%
\label{fig:ent}
\end{figure}

\textbf{Experimental Results.} We now test the IPA resulting from the entropy-based quality function \ref{quality}, which we name ENT-DPP. All our experiments use $\alpha=0.5$ as we found performance to be unaffected by all but extreme choices of $\alpha$ (i.e., those very close to $0$ or $1$). In Figure \ref{fig:ent}, we repeat all the single-objective benchmark experiments presented in the main body of the paper. We see that, although providing a performance boost over existing IPA, our well-motivated ENT-DPP is not as performant as our (simpler) IMP-DPP. Indeed, the superiority of IMP-DPP over this carefully constructed ENT-DPP provides yet further justification for the practical use of IMP-DPP. He believe that the poor performance of ENT-DPP is due to the required approximations, which degrade as we increase the number of inducing points. More precisely, the bound $H(\textbf{y}_Z;f^*)\leq\sum_{i=1}^MH(y_{\textbf{z}_i};f^*)$ becomes looser as we increase the number $M$ of inducing points $\textbf{z}_i$ and they become closer together (on average).

\section{DPP-IMP Ablation Study}
\label{appendix:ablation}

In order to justify the quality function used for our proposed IMP-DPP IPA strategy, we now perform an ablation study where we demonstrate the robustness of this IPA to small changes to its quality function. We compare the IPA resulting from the proposed quality function (the expected improvement with respect to  the minimal predicted value of the function) with some other reasonable choices. In particular, Figure \ref{fig:abl} demonstrates that IMP-DPP still outperforms existing IPA methods even when swapping its ReLU with a softmax approximation or when using the mean predicted value of the function instead of the minimal predicted value as its baseline.

\begin{figure}%
\centering
\includegraphics[width = 0.6\textwidth]{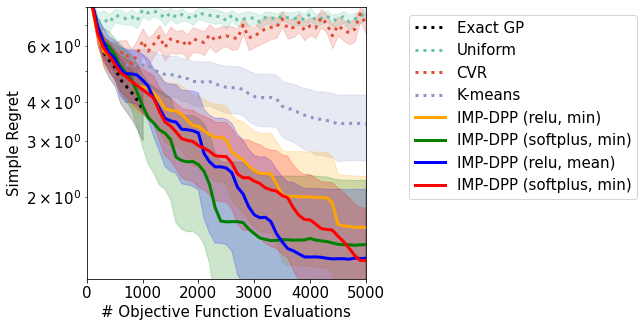}%
\caption{\label{fig:abl} Regret achieved on the 4d Shekel function for an exact GP and SVGPs with $M=250$ inducing points. All four variants of our IMP-DPP outperfom exisiting IPA strategies and achieve statistically similar results.
}
\end{figure}

\section{Additional Experimental Details}
\label{appendix:experiments}

\textbf{Synthetic baselines.} For details on the Shekel, Hartmann and Ackley functions see Appendix C of \cite{vakili2021scalable}, and for the multi-objective ZDT3 see \cite{zitzler2000comparison}.

We now present the analytical forms of the remaining Michalewicz and Rosenbrock functions.

\textbf{Michalewicz function.} A five-dimensional function with $5!$ local minima and a single global minimum defined on $\mathcal{X}\in[0,\pi]^5$:
\begin{align*}
    f(\textbf{x}) = -\sum_{i=1}^d\sin(x_i)\sin^{20}(\frac{ix_i^2}{\pi}).
\end{align*}

\textbf{Rosenbrock function.} A unimodal four-dimensional function a single global minimum lying in a narrow valley defined on $\mathcal{X}\in[-5,10]^5$:
\begin{align*}
    f(\textbf{x}) = \sum_{i=1}^3\left[100(x_{i+1}-x_i^2) + (x_i-1)^2\right].
\end{align*}

\textbf{Heat exchanger design.} For our final experiment, we considered the real-world constrained multi-objective problem of designing a both effective and light-weight heat exchanger (radiator) for a car, as parameterised by its dimensions and internal geometry.  Performance is simulated using areo-thermal analysis as part of an expensive-to-evaluate and highly non-linear digital twin. There are 6 continuous and 3 discrete inputs and three constraints that ensure a minimal standard of performance and safety. All our models use Mat\'{e}rn-5/2 kernels and a constant mean function. See \cite{paleyes2022penalisation} for further details about this problem.

\section{Wall-clock Timings}
\label{wallclock}

\textbf{6d Hartmann.} Figure \ref{fig:time} demonstrates the significant reduction in overhead (i.e., the computation required to generate the next set of query points) provide by using SVGPs instead of their exact counterparts. The experiments for Figure \ref{fig:time} were performed on a quad core 2.40GHz Intel Xeon CPU. We highlight that the GP is more expensive than the SVGP even when $N<M$. This perhaps surprising observation is due to the different optimisers used which have different early-stopping tolerances, i.e. the GP uses LBFGS whereas the SVGP uses Adam. Although the GP could likely be made a bit cheaper through appropriate tuning of the LBFGS stopping criterion, it will still be significantly more expensive than the SVGP when $M<<N$.

\begin{figure}[h]%
\begin{center}
    \includegraphics[width= 0.6\columnwidth]{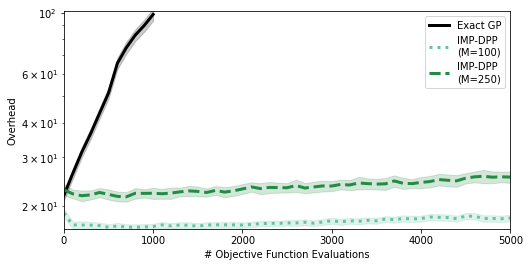}%
\end{center}
\caption{ The computational overhead incurred by using different surrogate models when optimising the 6d Hartmann function.  We see that our SVGP-based approaches have significantly lower overheads than the prohibitively expensive exact GP. Moreover, unlike the exact GP, the SVGP approaches maintain a constant overhead as BO progresses. 
}%
\label{fig:time}
\end{figure}

\textbf{Heat exchanger design.} In the main body of the paper we showed that an SVGP surrogate model using our proposed IPA and only $100$ inducing points is able to find a comparable Pareto front to an expensive exact GP. We now present additional insights through Figure \ref{fig:wallclock} which shows the total time spent running BO for each experiment. These wall-clock timings only measure the time spent fitting models and maximising acquisition functions, not the time spent evaluating the heat exchanger simulator. Figure \ref{fig:wallclock} demonstrates that the SVGP incurs order-of-magnitude lower optimisation overheads than the exact GP, opening up the feasibility of using BO on only moderately expensive functions, not just those with function query costs sufficiently large to overshadow the very significant optimisation overheads incurred by the exact GP models. These experiments were performed on a single NVIDIA GeForce GTX 1070 GPU. 

\begin{figure}%
\centering\includegraphics[width= 0.6\columnwidth]{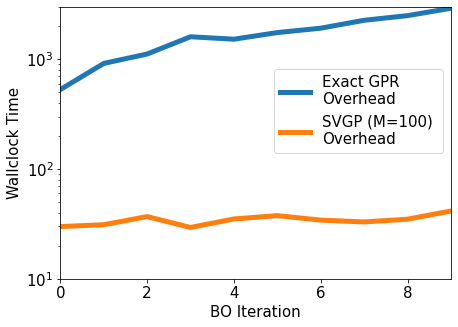}%
\caption{\label{fig:wallclock} Time taken (on a log scale) to fit the surrogate models and maximise acquisition functions for each BO iteration of our heat exchanger design problem. BO with an SVGP incurs order of magnitude lower costs than BO with an exact GP. Note also that the exact GP gets increasingly expensive as the optimisation progresses, whereas the SVGP's overhead remains roughly the same.
}
\end{figure}

\end{document}